\documentclass{article}
\usepackage{microtype}
\usepackage{graphicx}
\usepackage{subcaption}
\usepackage{booktabs}
\usepackage{hyperref}

\usepackage[preprint]{icml2026}
\usepackage{amsmath,amssymb,mathtools,amsthm}
\newtheorem{proposition}{Proposition}
\usepackage[capitalize,noabbrev]{cleveref}
\usepackage{tikz}
\usetikzlibrary{shapes.geometric, arrows.meta, positioning, calc}
\usepackage{xcolor}
\usepackage{bm}
\newcommand{\vect}[1]{\bm{#1}}
\usepackage{algorithm}
\usepackage{algorithmic}
\definecolor{inputbg}{HTML}{E8E0F0}
\definecolor{extractbg}{HTML}{D4E8F7}
\definecolor{dagbg}{HTML}{D5EDDA}
\definecolor{z3bg}{HTML}{FFF3CD}
\definecolor{resultbg}{HTML}{F8D7DA}
\definecolor{exactc}{HTML}{009E73}
\definecolor{approxc}{HTML}{E69F00}
\definecolor{arrowc}{HTML}{555555}

\icmltitlerunning{Compiling Fourier Neural Operators into SMT Solvers}

\begin{document}
\twocolumn[
  \icmltitle{Can We Formally Verify Neural PDE Surrogates?\\SMT Compilation of Small Fourier Neural Operators}
  \icmlsetsymbol{equal}{*}
  \begin{icmlauthorlist}
    \icmlauthor{Ali Baheri}{rit}
    \icmlauthor{Ignacio Laguna Peralta}{llnl}
  \end{icmlauthorlist}
  \icmlaffiliation{rit}{Rochester Institute of Technology, Rochester, NY, USA}
  \icmlaffiliation{llnl}{Lawrence Livermore National Laboratory, Livermore, CA, USA}
  \icmlcorrespondingauthor{Ali Baheri}{akbeme@rit.edu}
  \icmlkeywords{Neural Operators, Formal Verification, PDE Surrogates, SMT Solving}
  \vskip 0.3in
]
\makeatletter
\renewcommand{\Notice@String}{}
\makeatother
\printAffiliationsAndNotice{}

\begin{abstract}
Fourier Neural Operators (FNOs) can greatly accelerate PDE simulation, but they are often used without formal guarantees that they preserve basic physical structure. We show that, once the trained weights and grid are fixed, the spectral convolution in an FNO is a linear map. As a result, the full forward pass is piecewise-linear and can be represented exactly in Z3's linear real arithmetic. We study two encodings. The \emph{exact} encoding compiles the spectral convolution into a dense matrix multiplication, which is sound for both proofs and counterexamples. The lighter \emph{frozen} encoding replaces the spectral path with a constant, making it faster but approximate. On 10 small FNO surrogates for 1D advection-diffusion-reaction (85 to 117 parameters, grids 8 to 32), the exact encoding gives 2 sound positivity proofs on linear (ReLU-free) models, 5 sound positivity counterexamples, and 10 sound mass-violation counterexamples; the remaining 3 positivity queries on ReLU models time out. For mass non-increase, Z3 finds worse counterexamples than both gradient-based falsification and Monte Carlo on 7 of 10 models. The frozen encoding scales to grid size 64 with sub-second positivity checks, but it no longer provides certificates for the original FNO. Overall, the results make the soundness--scalability tradeoff explicit and point to what is needed for formal verification of production-scale neural operators.
\end{abstract}

\section{Introduction}
\label{sec:intro}

Learned PDE surrogates, including Fourier Neural Operators~\citep{li2021fno} and DeepONets~\citep{lu2021deeponet}, can reduce the cost of scientific simulation by orders of magnitude. The remaining obstacle is trust. A surrogate may silently break conservation, produce negative concentrations, or introduce other nonphysical behavior, and these errors can then propagate through downstream analyses. Most validation pipelines rely on Monte Carlo (MC) testing, which only checks the inputs that were sampled. Formal verification offers a stronger alternative: an SMT solver can either certify that a property holds for every input in a prescribed region or return a concrete counterexample.

Neural network verification is a mature and active area~\citep{katz2017reluplex, wang2021betacrown}, but most existing tools are designed for classifiers or control policies. They do not directly address spectral convolutions, function-space inputs, or properties such as conservation and positivity. This paper begins from a simple observation with useful consequences: on a fixed grid, and with trained weights held fixed, the spectral convolution in an FNO is a composition of linear maps--the DFT, a learned mode-wise scaling, and the inverse DFT. Therefore, an FNO with ReLU nonlinearities is piecewise-linear and can be compiled into Z3~\citep{demoura2008z3} without approximating the forward pass.

Our contributions are as follows. First, we give an exact compilation of fixed-grid FNOs into Z3's linear real arithmetic by constructing explicit spectral matrices. Second, to our knowledge, we give the first sound formal proofs for a neural PDE operator: positivity certificates for linear FNOs at grid sizes up to 16. Third, we compare Z3 with gradient-based falsification and MC across 10 small models. Finally, we characterize the soundness--scalability boundary: exact encodings support proofs for small linear models and counterexamples for somewhat larger ReLU models, while approximate encodings scale further but lose formal soundness. Although we use Z3 throughout, the compiled piecewise-linear representation is solver-agnostic and could be paired with specialized SMT, nonlinear, or neural-network verification tools in future work.

\section{Related Work}
\label{sec:related}

\textbf{Neural PDE surrogates.}
FNOs~\citep{li2021fno} learn integral kernels in Fourier space, while DeepONets~\citep{lu2021deeponet} represent operators through branch-trunk decompositions. Physics-informed networks~\citep{raissi2019pinns} add PDE residuals to the training objective, and neural operator theory~\citep{kovachki2023neural} gives approximation results for broad classes of operators. Together, these methods have made learned PDE surrogates increasingly practical. However, good predictive accuracy alone does not ensure that a trained surrogate respects physical constraints such as positivity or conservation.

\textbf{Neural network verification.}
Formal verification of neural networks has largely focused on classifiers and control policies. Reluplex~\citep{katz2017reluplex} introduced SMT-style reasoning for ReLU networks, and $\alpha$-$\beta$-CROWN~\citep{wang2021betacrown} scales bound propagation to large vision models. Related work in control and autonomy has studied safety validation and falsification for learning-enabled systems~\citep{shahrooei2023falsification,yancosek2024beacon,baheri2022verification}. These methods share our goal of finding or ruling out unsafe behavior, but they do not directly target spectral convolutions, function-space inputs, or PDE-specific properties.

\textbf{Trustworthy surrogates.}
Physics-constrained architectures~\citep{trask2022physics} reduce violations by building constraints into the model or the loss. Interval bound propagation~\citep{gowal2018effectiveness} can certify output ranges for standard feedforward networks, while LyZNet~\citep{liu2024lyznet} verifies Lyapunov stability and BEACONS~\citep{gorard2026beacons} provides algebraic error bounds. To the best of our knowledge, prior work has not compiled a trained Fourier neural operator into an exact SMT representation for verifying physics-motivated properties over a continuous input set.

\section{Method}
\label{sec:method}

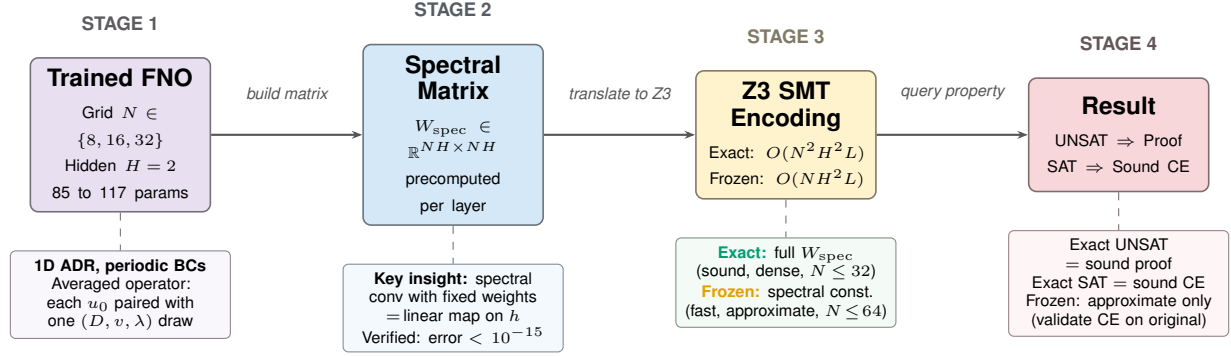
\begin{figure*}[t]
\centering
\begin{tikzpicture}[>=Stealth, font=\sffamily\footnotesize,
    mainbox/.style={rectangle, rounded corners=3pt, draw=#1!50!black, line width=0.6pt, fill=#1, minimum width=2.3cm, minimum height=1.5cm, text width=2.1cm, align=center, inner sep=4pt},
    detailbox/.style={rectangle, rounded corners=2pt, draw=#1!60!black, line width=0.4pt, fill=#1!25, text width=2.7cm, align=center, inner sep=3pt, font=\sffamily\tiny},
    arrmark/.style={-{Stealth[length=4pt,width=3pt]}, line width=0.8pt, color=arrowc}]
\node[mainbox=inputbg] (B1) {\textbf{\small Trained FNO}\\[2pt]{\tiny Grid $N\!\in\!\{8,16,32\}$}\\{\tiny Hidden $H\!=\!2$}\\{\tiny 85 to 117 params}};
\node[mainbox=extractbg, right=2.0cm of B1] (B2) {\textbf{\small Spectral}\\{\textbf{\small Matrix}}\\[2pt]{\tiny $W_{\mathrm{spec}} \in \mathbb{R}^{NH \times NH}$}\\{\tiny precomputed per layer}};
\node[mainbox=z3bg, right=2.0cm of B2] (B3) {\textbf{\small Z3 SMT}\\{\textbf{\small Encoding}}\\[2pt]{\tiny Exact: $O(N^2 H^2 L)$}\\{\tiny Frozen: $O(N H^2 L)$}};
\node[mainbox=resultbg, right=2.0cm of B3] (B4) {\textbf{\small Result}\\[2pt]{\tiny UNSAT $\Rightarrow$ Proof}\\{\tiny SAT $\Rightarrow$ Sound CE}};
\draw[arrmark] (B1.east) -- (B2.west);
\draw[arrmark] (B2.east) -- (B3.west);
\draw[arrmark] (B3.east) -- (B4.west);
\node[font=\sffamily\tiny\itshape, color=arrowc, above] at ($(B1.east)!0.5!(B2.west)+(0,0.35)$) {build matrix};
\node[font=\sffamily\tiny\itshape, color=arrowc, above] at ($(B2.east)!0.5!(B3.west)+(0,0.35)$) {translate to Z3};
\node[font=\sffamily\tiny\itshape, color=arrowc, above] at ($(B3.east)!0.5!(B4.west)+(0,0.35)$) {query property};
\node[font=\sffamily\scriptsize\bfseries, color=inputbg!40!black] at ($(B1.north)+(0,0.45)$) {STAGE 1};
\node[font=\sffamily\scriptsize\bfseries, color=extractbg!40!black] at ($(B2.north)+(0,0.45)$) {STAGE 2};
\node[font=\sffamily\scriptsize\bfseries, color=z3bg!50!black] at ($(B3.north)+(0,0.45)$) {STAGE 3};
\node[font=\sffamily\scriptsize\bfseries, color=resultbg!40!black] at ($(B4.north)+(0,0.45)$) {STAGE 4};
\node[detailbox=inputbg, below=0.5cm of B1] (D1) {\textbf{1D ADR, periodic BCs}\\Averaged operator:\\each $u_0$ paired with\\one $(D,v,\lambda)$ draw};
\draw[dashed, thin, color=inputbg!50!black] (B1.south) -- (D1.north);
\node[detailbox=extractbg, below=0.5cm of B2] (D2) {\textbf{Key insight:} spectral\\conv with fixed weights\\$=$\,linear map on $h$\\Verified: error $< 10^{-15}$};
\draw[dashed, thin, color=extractbg!50!black] (B2.south) -- (D2.north);
\node[detailbox=dagbg, below=0.5cm of B3] (D3) {\textcolor{exactc}{\textbf{Exact:}} full $W_{\mathrm{spec}}$\\(sound, dense, $N\!\leq\!32$)\\[1pt]\textcolor{approxc}{\textbf{Frozen:}} spectral const.\\(fast, approximate, $N\!\leq\!64$)};
\draw[dashed, thin, color=dagbg!50!black] (B3.south) -- (D3.north);
\node[detailbox=resultbg, below=0.5cm of B4] (D4) {Exact UNSAT $=$ sound proof\\Exact SAT $=$ sound CE\\Frozen: approximate only\\(validate CE on original)};
\draw[dashed, thin, color=resultbg!50!black] (B4.south) -- (D4.north);
\end{tikzpicture}
\vspace{-0.5em}
\caption{The FNO-to-SMT bridge. We compile the spectral convolution into an explicit dense matrix and then query physics-motivated properties with an SMT solver. The exact encoding is sound but dense; the frozen encoding is faster but approximate.}
\label{fig:pipeline}
\end{figure*}

Consider a single FNO hidden layer on a grid of size $N$ with hidden width $H$. Let $h\in\mathbb{R}^{N\times H}$ denote the hidden state. A spectral convolution applies a discrete Fourier transform, multiplies selected Fourier modes by learned complex weights, zeros or truncates the remaining modes, and then applies the inverse transform. Once the weights are fixed, this operation can be written as
\begin{equation}
  \mathcal{K}_{\ell}(h)=\operatorname{IFFT}\!\left(R_{\ell}\odot \operatorname{FFT}(h)\right),
\end{equation}
where $R_{\ell}$ contains the learned mode weights. Every operation in this expression is linear in $h$. Therefore, there is a real matrix $W^{\mathrm{spec}}_{\ell}\in\mathbb{R}^{NH\times NH}$ such that
\begin{equation}
  \vect(\mathcal{K}_{\ell}(h)) = W^{\mathrm{spec}}_{\ell}\vect(h).
\end{equation}
We build $W^{\mathrm{spec}}_{\ell}$ by applying the spectral convolution to each standard basis vector and stacking the resulting outputs as columns. In implementation, we validate this construction by comparing the dense matrix product against the FFT-based code path on random inputs; the largest observed discrepancy is below $10^{-15}$.

\begin{proposition}
A fixed-grid FNO with fixed learned parameters and ReLU nonlinearities is a piecewise-linear map from the discretized input function to the discretized output function.
\end{proposition}
\begin{proof}
The lifting layer, pointwise bypass maps, spectral matrices, and projection layer are affine maps over the real-valued grid representation. ReLU is piecewise-linear and is applied coordinate-wise. A finite composition of affine maps and coordinate-wise piecewise-linear maps is piecewise-linear.
\end{proof}

\subsection{Exact Encoding}

The exact encoding merges the spectral matrix and pointwise bypass map into one affine expression per layer. After vectorization, a hidden layer can be written as
\begin{equation}
  z^{(\ell)} = \left(W^{\mathrm{spec}}_{\ell}+W^{\mathrm{byp}}_{\ell}\right)h^{(\ell-1)} + b_{\ell},
  \qquad
  h^{(\ell)} = \sigma(z^{(\ell)}),
\end{equation}
where $\sigma(t)=\max(t,0)$ for ReLU layers and $\sigma(t)=t$ for linear layers. In Z3, each ReLU coordinate is encoded as an if-then-else term, or equivalently as the graph of $\max(t,0)$ over the reals. For the compiled real-valued model, this encoding is sound: UNSAT proves that no admissible real input violates the property, while SAT returns a valid real counterexample.

The main cost is the dense spectral matrix. Each layer introduces $O(N^2H^2)$ linear terms, so a network with $L$ spectral layers contributes $O(N^2H^2L)$ terms before adding the property constraints. This is manageable for the small models studied here, but it is not a scalable representation for high-resolution FNOs.

\subsection{Frozen Encoding}

To explore larger grids, we also evaluate a cheaper approximation. Given a reference input $u^{\mathrm{ref}}$, we evaluate the spectral path at that reference and substitute the result as a constant:
\begin{equation}
  z^{(\ell)}_{ij} = \sum_{m=1}^{H} W^{\mathrm{byp}}_{\ell,jm}h^{(\ell-1)}_{im} + b_{\ell,j} + c^{(\ell)}_{ij}.
\end{equation}
This is not a first-order linearization, since the spectral path is already linear. Instead, it is a zeroth-order substitution that removes the input dependence of the spectral component. The term count drops to $O(NH^2L)$, but the SMT problem now describes the frozen surrogate rather than the original FNO. Consequently, frozen UNSAT results are not formal certificates for the original model. Frozen SAT results must be checked on the original FNO before they can be treated as genuine counterexamples.

\subsection{Properties}

Our experiments use the one-dimensional advection-diffusion-reaction equation with periodic boundary conditions,
\begin{equation}
  \partial_t u + v\partial_x u = D\partial_{xx}u - \lambda u .
\end{equation}
Each training example pairs a random initial condition with one draw of $(D,v,\lambda)$. Thus, the surrogate is an averaged, parameter-free operator, rather than a model conditioned on the PDE coefficients.

Let $f_\theta$ be the trained FNO, and define the discrete mass as
\begin{equation}
  M(u)=\frac{1}{N}\sum_{i=1}^{N}u_i .
\end{equation}
We study two properties.

\paragraph{Mass non-increase.}
Under periodic boundary conditions and a nonnegative reaction rate, the physical mass satisfies $dM/dt=-\lambda M(t)\leq 0$ for nonnegative states. We therefore search for violations of
\begin{equation}
  M(f_\theta(u)) \leq M(u)+\varepsilon,
\end{equation}
with $\varepsilon=0.05$. A counterexample satisfies $M(f_\theta(u))-M(u)>\varepsilon$.

\paragraph{Positivity.}
For nonnegative initial conditions, the PDE should not create negative concentrations. We verify whether
\begin{equation}
  \min_i f_\theta(u)_i \geq 0
\end{equation}
for all smooth inputs with $u_i\geq 0.1$. A counterexample satisfies $f_\theta(u)_i<0$ for at least one grid point.

For both properties, admissible inputs are constrained by
\begin{equation}
  0\leq u_i\leq 5,\qquad |u_{i+1}-u_i|\leq \frac{15}{N},\qquad |u_1-u_N|\leq \frac{15}{N},
\end{equation}
with the stronger lower bound $u_i\geq 0.1$ for positivity queries. These constraints define a continuous family of smooth periodic inputs, not a finite test set.

\section{Experiments}
\label{sec:experiments}

\subsection{Setup}

We train 10 FNO surrogates on 1D ADR with periodic boundary conditions at $N \in \{8, 16, 32\}$, hidden width $H{=}2$, and two depths: $L{=}1$ (linear, 6 models) and $L{=}2$ (one ReLU layer, 4 models). The models have 85 to 117 parameters. Training uses random-feature regression: hidden layers are frozen, and the projection layer is fit by least squares on 500 samples. This is \textit{not} standard end-to-end training. The compilation method itself is architecture-agnostic, but our scalability demonstration is limited to these toy-sized models.

\textbf{Baselines.} We compare against two baselines. The first is MC testing with 5,000 random smooth inputs. The second is gradient-based falsification: projected finite-difference ascent on the original model, using 10 random restarts and 100 steps per restart. We use finite differences rather than autograd because the implementation is in numpy. A PyTorch implementation with Adam and multiple restarts would be a stronger baseline, and we leave that comparison to future work. Both baselines operate on the original model, not on an encoding.

\begin{figure}[t]
\centering
\includegraphics[width=\linewidth]{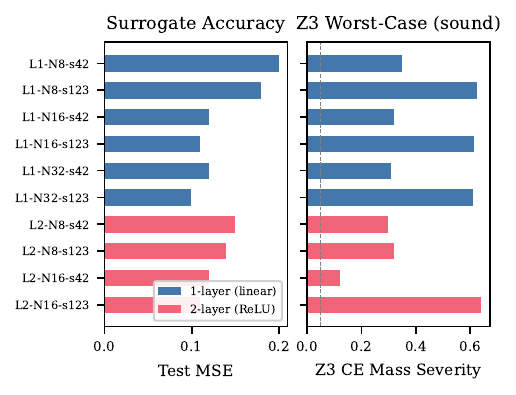}
\vspace{-1.5em}
\caption{Model quality. \textbf{Left:} Test MSE. \textbf{Right:} Severity of Z3's sound counterexample for mass non-increase, with every counterexample re-evaluated on the original model.}
\label{fig:quality}
\end{figure}

\subsection{Mass Non-Increase}

With the exact encoding, Z3 finds sound mass-violation counterexamples on \textbf{all 10 models} at $\varepsilon{=}0.05$ (\Cref{fig:mass}). Each counterexample is verified on the original model with zero approximation gap, and the violation severity ranges from 0.12 to 0.64.

Z3 finds the most severe counterexample on 7 of 10 models (\Cref{fig:mass}). Its violations are 1.1 to 1.6$\times$ more severe than those found by gradient falsification and 1.2 to 2.5$\times$ more severe than those found by MC. Gradient search is competitive on the 2-layer models (severity 0.58 to 0.66 versus Z3's 0.12 to 0.64), suggesting that the ReLU nonlinearity creates a landscape where local search can find deep violations. On the 1-layer models, Z3 has a clearer advantage because the linear structure lets the solver optimize globally. Z3 solve times for mass range from $<$0.1\,s at $N{=}8$ to 1.8\,s for the $N{=}16$, $L{=}2$ case. Gradient falsification takes 1 to 4\,s.

\begin{figure}[t]
\centering
\includegraphics[width=\linewidth]{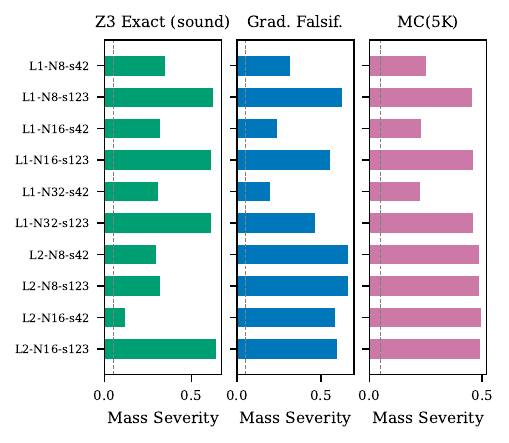}
\vspace{-1.5em}
\caption{Mass non-increase ($\varepsilon{=}0.05$): severity of the worst counterexample found by Z3 exact (sound), gradient falsification, and MC(5K). Z3 finds the most severe CE on 7 of 10 models.}
\label{fig:mass}
\end{figure}

\subsection{Positivity: Sound Proofs and Counterexamples}

The positivity results are more nuanced (\Cref{fig:positivity}, \Cref{tab:summary}). On 1-layer linear models, Z3 produces \textbf{2 sound positivity proofs}: L1-N8-s123 (0.1\,s) and L1-N16-s123 (2.3\,s). These results certify that, for \textit{all} smooth positive inputs with $u_i \geq 0.1$, the output is nonnegative at every grid point. To our knowledge, these are the first sound formal proofs for a neural PDE operator. At the same time, the scope is deliberately narrow: the certificates apply to the compiled real-valued versions of very small, linear (ReLU-free) FNOs. Z3 also finds \textbf{4 sound positivity counterexamples} on the remaining linear models.

On 2-layer ReLU models, Z3 finds \textbf{1 sound counterexample} (L2-N8-s42, minimum output $-0.063$) and \textbf{times out} on the other 3 positivity queries. The $NH{=}16$ ReLU neurons allow up to $2^{16}$ activation patterns, which creates a difficult case-splitting problem. Gradient falsification finds violations on all 4 ReLU models, with minimum outputs from $-0.088$ to $-0.577$. In these cases, gradient search is often more severe than Z3, showing that it remains a strong adversary for differentiable models.

\begin{figure}[t]
\centering
\includegraphics[width=\linewidth]{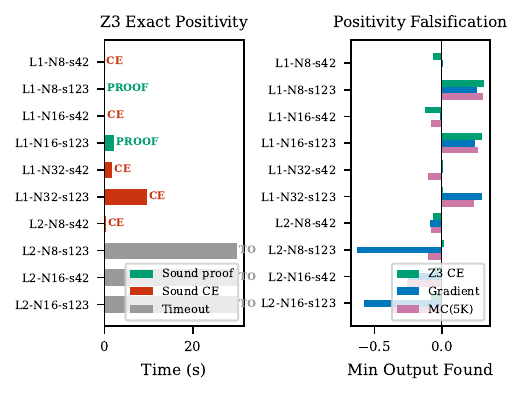}
\vspace{-1.5em}
\caption{Positivity. \textbf{Left:} Z3 exact results: green = sound proof, red = sound CE, gray = timeout. \textbf{Right:} Minimum output found by Z3, gradient falsification, and MC. Z3 proofs certify positivity for all constrained inputs.}
\label{fig:positivity}
\end{figure}

\subsection{Frozen Encoding: Trading Soundness for Scale}

To make the tradeoff concrete, we apply the frozen encoding to 8 models at $N{=}32\text{--}64$ with $H \in \{2,4,8\}$ (117 to 3,249 parameters). The frozen encoding certifies positivity on all 8 frozen surrogates in 0.1 to 2.2\,s and finds mass counterexamples on 4 of 8 models at $\varepsilon{=}0.05$ in 0.03 to 7.5\,s. Of those mass counterexamples, 3 are confirmed on the original model (severity 0.25 to 0.63) and 1 is not confirmed (severity 0.02 $<$ $\varepsilon$). The positivity proofs are approximate: they certify only the frozen surrogate, not the original FNO. None of the $N{=}32\text{--}64$ original models violate positivity on the test set, so the frozen proofs agree with empirical evidence here, but they should not be read as formal guarantees for the original models.

\subsection{The Soundness--Scalability Boundary}

\Cref{tab:summary} and \Cref{fig:solvetime} summarize the boundary. The exact encoding provides: (1) sound mass counterexamples on all models up to $N{=}32$ in less than 10\,s; (2) sound positivity counterexamples on linear and small ReLU models; and (3) sound positivity proofs on linear models up to $N{=}16$ in less than 3\,s. It does \textit{not} provide sound proofs for the ReLU models within the timeout. The frozen encoding reaches $N{=}64$ and larger hidden widths, but it does so by giving up soundness for the original model.

\begin{table}[t]
\centering
\small
\caption{Summary. Exact encoding at $\varepsilon{=}0.05$; all CEs are sound. Frozen encoding on separate models at $N{=}32\text{--}64$.}
\label{tab:summary}
\vspace{0.3em}
\begin{tabular}{@{}llcccc@{}}
\toprule
Encoding & Property & Proof & CE & UC & TO \\
\midrule
Exact, $L{=}1$ & Mass & 0 & 6 & -- & 0 \\
Exact, $L{=}1$ & Positivity & 2 & 4 & -- & 0 \\
Exact, $L{=}2$ & Mass & 0 & 4 & -- & 0 \\
Exact, $L{=}2$ & Positivity & 0 & 1 & -- & 3 \\
\midrule
Frozen, $N{\leq}64$ & Mass & 0 & 3 & 1 & 4 \\
Frozen, $N{\leq}64$ & Positivity & 8* & 0 & -- & 0 \\
\bottomrule
\multicolumn{6}{@{}l}{\footnotesize *Certifies frozen surrogate only, not original model.}
\end{tabular}
\end{table}

\begin{figure}[t]
\centering
\includegraphics[width=\linewidth]{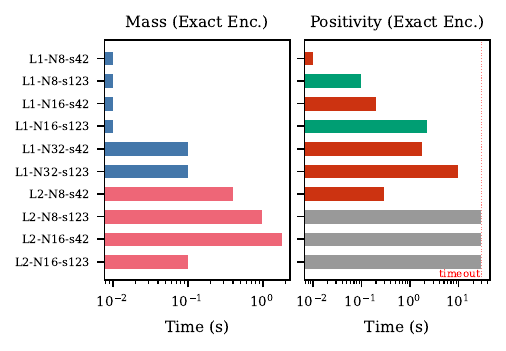}
\vspace{-1.5em}
\caption{Z3 solve time with the exact encoding. \textbf{Left:} Mass. \textbf{Right:} Positivity. Linear models solve faster, while ReLU models time out for positivity proofs.}
\label{fig:solvetime}
\end{figure}

\section{Discussion}
\label{sec:discussion}

\textbf{What this demonstrates.} The main technical point is that FNOs can be compiled exactly into piecewise-linear form through spectral matrix construction. On small models, this lets Z3 produce both sound proofs and sound counterexamples. In particular, we obtain 2 positivity proofs for linear models and mass-violation counterexamples for every model tested. For mass non-increase, Z3 also finds more severe counterexamples than gradient falsification on 7 of the 10 models.

\textbf{What this does not demonstrate.} These experiments do not solve FNO verification at scale. The models are small: 85 to 117 parameters, 1D ADR data, 500 training samples, and frozen hidden layers. Sound proofs are only obtained for linear models ($L{=}1$); ReLU positivity proofs time out. A production FNO with $N \geq 128$ and $H \geq 32$ would produce spectral matrices of dimension $4096 \times 4096$ or larger, which is far beyond the regime where the current Z3 encoding is practical.

\textbf{Gradient falsification is a strong adversary.} On 2-layer models, gradient search finds positivity violations that are comparable to, and sometimes more severe than, those found by Z3 at lower cost. Our finite-difference implementation is also not the strongest possible version of this baseline. Autograd-based projected gradient ascent with Adam and many random restarts would likely be stronger. Z3's distinct value is not that it is always the cheapest falsifier, but that it can prove the absence of violations and can optimize globally on linear models.

\textbf{Path to scale.} The spectral matrix construction is the enabling step; the downstream verifier can change. Once an FNO has been compiled into a piecewise-linear network, specialized neural-network verification tools such as $\alpha$-$\beta$-CROWN or bound-propagation methods may scale better than a general-purpose SMT solver. SMT over quantifier-free linear real arithmetic (QF\_LRA) is decidable, unlike general integer arithmetic, but the ReLU case splits make the piecewise-linear fragment NP-complete, equivalent in spirit to mixed-integer programming. This explains why Z3 times out on ReLU positivity proofs. Solvers specialized for this fragment, such as Marabou or $\alpha$-$\beta$-CROWN, exploit network structure that Z3 does not. Exact rational arithmetic would also help close the gap between real-valued reasoning and floating-point implementation. Extending the bridge from random-feature regression to gradient-trained FNOs requires no conceptual change, since only the learned weight values enter the encoding, but larger $H$ and $N$ will require more scalable solvers.

\section{Conclusion}
\label{sec:conclusion}

We showed that small FNOs can be compiled exactly into piecewise-linear SMT form by constructing explicit spectral matrices. This yields, to our knowledge, the first sound formal proofs for a neural PDE operator: positivity certificates for very small linear FNOs at grids up to 16. The exact encoding also finds sound counterexamples that are more severe than gradient-based falsification on most models. The frozen encoding reaches larger grids, but only by giving up soundness for the original FNO. The boundary is therefore clear: exact proofs are currently practical for small linear models, exact counterexamples extend to moderate ReLU models, and larger neural operators will require either approximate encodings or integration with scalable neural-network verifiers. We hope this proof of concept encourages the formal verification community to engage with neural PDE surrogates, and encourages the operator-learning community to treat verifiability as a design criterion. Code and data will be released upon publication.

\bibliography{references}
\bibliographystyle{icml2026}

\end{document}